%% file: main.tex
\documentclass[]{article}

\include{preamble}
\include{math}

\title{On the Convergence of\\ Discounted Policy Gradient Methods}
\author{Chris Nota \\ \small Autonomous Learning Laboratory \\ \small University of Massachusetts Amherst \\ \small cnota@cs.umass.edu}
\date{}

\begin{document}

\maketitle

\begin{abstract}

Many popular policy gradient methods for reinforcement learning follow a biased approximation of the policy gradient known as the discounted approximation.
While it has been shown that the discounted approximation of the policy gradient is not the gradient of any objective function, little else is known about its convergence behavior or properties.
In this paper, we show that if the discounted approximation is followed such that the discount factor is increased slowly at a rate related to a decreasing learning rate, the resulting method recovers the standard guarantees of gradient ascent on the undiscounted objective.
\end{abstract}

\section{Introduction}
\emph{Policy gradient methods} are a class of \emph{reinforcement learning} (RL) algorithms that attempt to directly maximize the expected performance of an agent's \emph{policy} by following the gradient of an objective function \citep{sutton2000policy}, typically the expected sum of rewards, using a stochastic estimator generated by interacting with the environment.
Unbiased estimators of this gradient can suffer from high variance due to high variance in the sum of future rewards.
A common approach is to instead consider an exponentially \emph{discounted} sum of future rewards.
This approach reduces the variance of most estimators but introduces bias \citep{thomas2014bias}.
Frequently, the discounted sum of future rewards is estimated by a \emph{critic} \citep{konda2000actor}.
It has been argued that when a critic is used, discounting has the additional benefit of reducing \emph{approximation error} \citep{zhang2020deeper}.

The ``discounted'' policy gradient was originally introduced as the gradient of a discounted objective \citep{sutton2000policy}.
However, it has been shown that the gradient of the discounted objective does not produce the update direction followed by most discounted policy gradient algorithms \citep{thomas2014bias, nota2019policy}.
Instead, most algorithms follow a direction sometimes called the ``discounted approximation'' of the policy gradient.\footnote{The earliest use of this terminology that we could find was by \citet{schulman2015high}.}
It has been shown that the discounted approximation is not the gradient of \emph{any} objective function \citep{nota2019policy}.
This raises the question of how exactly the discounted approximation should be interpreted, and under what circumstances following this direction leads to the optimal policy.

In this paper, we show that the discounted approximation of the policy gradient can be viewed as a biased approximation of the undiscounted objective.
We show that the bias can be computed in closed form and its magnitude is upper bounded by a value proportional to $(1 - \gamma)$, where $\gamma \in [0, 1]$ is the \emph{discount factor}.
We then show by applying standard results for the convergence of gradient methods with errors \citep{bertsekas2000gradient} that by slowly increasing $\gamma \to 1$ at a rate inversely proportionally to an adequately decaying step size, the resulting policy gradient method will converge to a locally optimal policy.

\section{Background}

\subsection{Notation}

In RL, the environment is typically expressed as a \emph{Markov decision process} (MDP).
An MDP is a tuple, $(\mathcal S, \mathcal A, P, R, d_0)$, 
where $\mathcal S$ is the set of possible \emph{states} of the environment, 
$\mathcal A$ is the set of \emph{actions} available to the agent, 
$P: \mathcal S \times \mathcal A \to \mathcal D(\mathcal S)$ is a \emph{transition function} that determines the probability distribution over the next state given the current state and action, 
$R: \mathcal S \times \mathcal A \times \mathcal S \to \mathcal D([-R_\text{max}, R_\text{max}])$ is the distribution over rewards given a transition, bounded by some maximum reward $R_\text{max} \in \mathbb R$, 
and $d_0: \mathcal S \to [0, 1]$ is the \emph{initial state distribution}.

An \emph{episode} begins at timestep $0$ and terminates no later than time $T$.
For each episode, an initial state, $S_0$, is sampled from $d_0$.
At each timestep $t$, the agent observes state $S_t$, selects an action $A_t$, transitions to the next state $S_{t+1} \sim P(\cdot | S_t, A_t)$, and receives a reward $R_t \sim R(\cdot | S_t, A_t, S_{t+1})$.
To simplify the mathematical treatment, a special state called the \emph{terminal absorbing state} is often defined which represents the end of the episode; the agent is ``stuck'' in this state until time $T$ and receives no rewards.
The episodic setting is the default choice for many practical applications.

Actions are selected by the agent according to a \emph{policy}, $\pi$, such that $A_t \sim \pi(\cdot | S_t)$.
$\pi_\theta$ is a \emph{parameterized policy}, such that $\theta$ is a vector of parameters which are optimized by the agent.
The \emph{objective} of the agent is to find the parameters which optimize the function $J$ given by

\begin{equation}
J(\theta) = \mathbb E\left[ \sum_{t=0}^{T-1} R_t \middle | \pi = \pi_\theta \right].
\end{equation}

Note that because $S_T$ is the terminal state, no reward is received at time $T$.
The \emph{state-value function}, $V^\pi$, gives the expected returns from starting in a particular state given a policy. The \emph{action-value function}, (or simply the ``Q-function''), $Q^\pi$, gives the expected returns from a state-action pair given a policy. They are given by:

\begin{align*}
V^\pi(S_t) = \mathbb E\left[\sum_{i=t}^{T-1} R_i \middle | S_t, \pi = \pi_\theta \right], & & Q^\pi(S_t, A_t) = \mathbb E\left[\sum_{i=t}^{T-1} R_i \middle | S_t, A_t, \pi = \pi_\theta \right].
\end{align*}

The \emph{discount factor}, $\gamma \in [0, 1]$, is scalar value that decreases the importance of future rewards relative to immediate rewards.
The \emph{discounted returns} from time $t$ are given by $\sum_{i=t}^T \gamma^{i - t} R_i$.
The \emph{discounted} value functions are then given by $V^\pi_\gamma$ and $Q^\pi_\gamma$.
They are defined as above, except the discounted returns are substituted for the undiscounted returns.

\subsection{Policy Gradient Methods}

\emph{Policy gradient methods} attempt to directly optimize $J$.
\cite{sutton2000policy} showed that the gradient of $J$ can be written in terms of $Q^{\pi_\theta}$:

\begin{equation}
    \nabla J(\theta) = \E \left [ \sum_{t=0}^{T-1} Q^{\pi_\theta}(S_t, A_t) \frac{\partial}{\partial \theta} \ln \pi_\theta (S_t, A_t)\middle | \pi = \pi_\theta\right].
\end{equation}

We assume that policy is Lipschitz continuous in that there exists some constant $L_\pi$ such that for all $s \in \mathcal S$ and $a \in \mathcal A,$ $\frac{\partial}{\partial \theta} \ln \pi_\theta(a | s) \leq L_\pi$.\footnote{While existing convergence proofs for policy gradient methods depend on this property \citep{wang2019neural}, it is rarely satisfied by standard neural network architectures. A remedy for this theory--practice gap is outside the scope of this paper. 
}
The \emph{discounted approximation} of the policy gradient substitutes the discounted action-value function into the expression above.
We define this approximation as

\begin{equation}
\hat \nabla (\theta, \gamma) \coloneqq \E \left [ \sum_{t=0}^{T-1} Q^{\pi_\theta}_\gamma(S_t, A_t) \frac{\partial}{\partial \theta} \ln \pi_\theta (S_t, A_t)\middle | \pi = \pi_\theta\right].
\end{equation}

This approximation has often been mistaken for the gradient of the discounted objective.\footnote{An incomplete review of incorrect uses of the discounted approximation in deep reinforcement learning was given by \citet{nota2019policy}.}
It has been shown that not only is this assumption incorrect \citep{thomas2014bias}, but the approximation is not the gradient of \emph{any} objective \citep{nota2019policy}. The correct gradient of the discounted objective was given by \citet{sutton2000policy}.

Prior work has argued that the discounted approximation is ``missing'' an extra $\gamma^t$ term that is found in the gradient of the discounted objective \citep{thomas2014bias}.
However, the discounted objective itself is not widely used and is sometimes considered ``deprecated'' \citep{sutton2018reinforcement} because it rarely reflects the true goals of practitioners in the episodic setting and is not well-defined in the continuing setting \citep{naik2019discounted}.

Therefore, in this paper we view the discounted approximation as a biased approximation of the undiscounted objective.
The use of the discounted approximation has traditionally been understood in terms of a bias-variance trade-off \citep{sutton2018reinforcement}, but \citet{zhang2020deeper} argued that the discounted approximation also helps combat the bias caused by approximating $Q^{\pi_\theta}$, resulting in a 3-way trade-off between bias, variance, and representation error. 

\subsection{The Convergence of Gradient Methods with Errors}

\citet{bertsekas2000gradient} provided several proofs extending and strengthening the convergence properties of gradient descent across a range of settings. In particular, they considered the convergence properties of sequences of the form\footnote{\citet{bertsekas2000gradient} use $t$ for the index of the sequence and $\gamma$ for the step size. We instead use $i$ for the index and $\alpha$ for the step size as $t$ and $\gamma$ are already used elsewhere.}

\begin{equation}
    x_{i+1} = x_i + \alpha_i(s_i + w_i),
\end{equation}

where $x_i$ is a parameter vector, $\alpha_i$ is a step size, $s_i$ is a descent direction for some objective function $f(x_i)$, and $w_i$ is a vector of errors.
For example, direct gradient ascent on $J(\theta)$ would be described by the sequence

\begin{equation}
    \theta_{i+1} = \theta_{i} + \alpha_i \nabla J(\theta),
\end{equation}

where $s_i = \nabla J(\theta)$ and $w_i$ is the zero vector.
The convergence results apply equally to ascent and descent directions; we will give the results in terms of ascent directions as policy gradient methods are typically described in terms of ascending $J$.
The proofs given by \citet{bertsekas2000gradient} require several assumptions, which we give below.

\begin{ass}
    \label{ass:lipschitz}
    $f: \mathbb R^n \to \mathbb R$ is a continuously differentiable scalar function on $\mathbb R^n$ such that for some constant $L$ we have
    \begin{equation}
    \forall x, \bar x \in \mathbb R^n: \Vert \nabla f(x) - \nabla f(\bar x) \Vert \leq L \Vert x - \bar x \Vert.
    \end{equation}
\end{ass}

\begin{ass}
    \label{ass:step-size}
    The step size $\alpha_i$ is positive and satisfies
    \begin{align}
        \sum_{i=0}^\infty \alpha_i = \infty, & & \sum_{i=0}^\infty \alpha_i^2 < \infty.
    \end{align}
\end{ass}

\begin{ass}
    \label{ass:ascent}
    $s_i$ is an ascent direction satisfying for some positive scalars $c_1$ and $c_2$:
    \begin{align}
        c_1 \Vert \nabla f(x_i) \Vert^2 \leq \nabla f(x_i) \cdot s_i, & & \Vert s_i \Vert \leq c_2 \Vert \nabla f(x_i) \Vert.
    \end{align}
\end{ass}

\begin{ass}
    \label{ass:error}
    $w_i$ is an error vector satisfying for some positive scalars $p$ and $q$:
    \begin{align}
        \Vert w_i \Vert \leq \alpha_i \big(p + q \Vert \nabla f(x_i) \Vert \big).
    \end{align}
\end{ass}

Assumption \ref{ass:error} will be the most interesting to us, constraining the magnitude of the error vector $w_i$.
Because this magnitude is proportional to $\alpha_i$, in the limit $w_i$ must decay to $0$.
Given the above assumptions, we have:

\begin{thm}
    \label{thm:convergence}
    Let $x_i$ be the sequence generated by the method
    \begin{equation}
        x_{i+1} = x_i + \alpha_t(s_i + w_i),
    \end{equation}
    satisfying Assumptions \ref{ass:lipschitz}-\ref{ass:error}. Then either $f(x_i) \to \infty$ or else $f(x_i)$ converges to a finite value and $\lim_{i \to \infty} \nabla f(x_i) = 0$. Furthermore, every limit point of $x_i$ is a stationary point of $f$. 
\end{thm}

The goal in this paper will be to decompose the discounted approximation of the policy gradient into an ascent direction and an error term, allowing us to apply Theorem \ref{thm:convergence}.
We will then show that Assumption \ref{ass:error} can be satisfied by appropriately decaying $(1 - \gamma)$.

\section{Bias in the Discounted Approximation}

In this section, we will show that the discounted approximation can be viewed as a biased approximation of the true policy gradient, and that the bias has an upper bound proportional to $(1 - \gamma)$.
\citet{nota2019policy} showed that the discounted approximation can be written as

\begin{equation}
\label{eq:discounted-approximation}
\hat \nabla (\theta, \gamma) = \sum_{s \in S} \ddf(s) \frac{\partial}{\partial \theta} \dvf (s),
\end{equation}

where

\begin{equation}
\ddf(s) \coloneqq d_0(s) + (1 - \gamma) \sum_{t=1}^{T-1} \Pr(S_t = s | \pi = \pi_\theta).
\end{equation}

We begin by relating this form to the gradient of the undiscounted objective, allowing us to better compare the two.
In other words, we will show how the \emph{undiscounted} objective, $J(\theta)$, can be written in terms of the \emph{discounted} value function.

\begin{lemma}
\label{lem:decomp}
For all $\gamma \in [0, 1]$:
\begin{equation}
     J(\theta) = \sum_{s \in \mathcal S} \ddf(s) \dvf(s).
\end{equation}
\end{lemma}
\begin{proof}
Consider that for any $\gamma$ and $t$, we can rearrange the Bellman equation \citep{sutton2018reinforcement}: $\mathbb E[R_t | \pi = \pi_\theta] = \mathbb E[\dvf(S_t) - \gamma \dvf(S_{t+1}) | \pi = \pi_\theta]$.
This allows us to rewrite the objective:

\begin{align*}
    \forall \gamma: &\mathbb E \left[\sum_{t=0}^{T-1} R_t \bigg| \pi = \pi_\theta \right] \\
    =& \mathbb E\left[\sum_{t=0}^{T-1} \left(\dvf(S_t) - \gamma \dvf(S_{t+1}) \right) \bigg| \pi = \pi_\theta \right] \\
    =& \mathbb E \left[\dvf(S_0) + \sum_{t=1}^{T-1} \left(\dvf(S_t) - \gamma \dvf(S_t)\right) - \underbrace{\gamma \dvf(S_T)}_{0} \bigg | \pi = \pi_\theta\right] \\
    =& \mathbb E \left[\dvf(S_0) + (1 - \gamma) \sum_{t=1}^{T-1}  \dvf(S_t)  \bigg | \pi = \pi_\theta \right] \\
    =& \sum_{s \in \mathcal S} \left(\Pr(S_0 = s) \dvf(s) +  (1 - \gamma) \sum_{t = 1}^{T-1}\Pr(S_t = s | \pi = \pi_\theta) \dvf(s) \right)  \\
    =& \sum_{s \in S} d^\theta_\gamma(s) \dvf(s).
\end{align*}
\end{proof}

Immediately, by differentiating the above expression we have:

\begin{cor}
\label{cor:grad}
For all $\gamma \in [0, 1]$:

\begin{equation}
\nabla J(\theta) = \sum_{s \in \mathcal S} \ddf(s) \frac{\partial}{\partial \theta} \dvf(s) + \sum_{s \in \mathcal S} \dvf(s) \frac{\partial}{\partial \theta}  \ddf(s).
\end{equation}
\end{cor}
\begin{proof}
This follows immediately from Lemma \ref{lem:decomp} and the product rule.
\end{proof}

By combining Equation \ref{eq:discounted-approximation} and Corollary \ref{cor:grad}, we can write $\hat \nabla(\theta, \gamma)$ as the sum of an ascent direction on the undiscounted objective, i.e., $\nabla J(\theta)$, and an error vector:

\begin{equation}
    \label{eq:bias}
    \hat \nabla(\theta, \gamma) = \underbrace{\nabla J(\theta)}_\text{ascent direction} - \underbrace{\sum_{s \in \mathcal S} \dvf(s) \frac{\partial}{\partial \theta}  \ddf(s)}_\text{error vector}.
\end{equation}

We can easily see that the error vector contains a coefficient of $(1 - \gamma)$:

\begin{equation}
\sum_{s \in \mathcal S} \dvf(s) \frac{\partial}{\partial \theta} \ddf(s) = (1 - \gamma) \sum_{s \in \mathcal S} \dvf(s) \frac{\partial}{\partial \theta} \sum_{t=1}^{T-1} \Pr(S_t = s | \pi = \pi_\theta).
\end{equation}

Notice that $d_0(s)$ is dropped from the above expression because $\frac{\partial}{\partial \theta} d_0(s) = 0$.
Due to the coefficient $(1-\gamma)$, if we show that the multiplicand is bounded, then the entire expression is bounded by a quantity proportional to $(1 - \gamma)$. By carefully choosing a sequence of discount factors, $\gamma_i$, we can then satisfy Assumption \ref{ass:error}.
We begin with a helpful lemma:

\begin{lemma}
For all $t$, there exists some finite Lipschitz constant $L_t$ such that for all $s \in \mathcal S$:
\begin{align}
\left \Vert \frac{\partial}{\partial \theta} \Pr(S_t = s | \pi = \pi_\theta) \right \Vert &\leq L_t.
\end{align}
\end{lemma}
\begin{proof}
Assume that for a given timestep $t$, for all $s$: $\left \Vert \frac{\partial}{\partial \theta} \Pr(S_{t-1} = s | \pi = \pi_\theta) \right \Vert \leq L_{t-1}$ for some positive constant $L_{t-1}$.
We will try to show that the constant $L_t$ therefore exists.
Because we only need to show that the constant exists, we do not need to worry about the tightness of the bound; even a very loose bound is sufficient.
Consider the $i$th parameter $\theta_i$. For all $s_t$:

\begin{align*}
&\frac{\partial}{\partial \theta_i} \Pr(S_t = s_t | \pi = \pi_\theta) \\ 
=& \frac{\partial}{\partial \theta_i} \sum_{s_{t-1} \in \mathcal S} \sum_{a_{t-1} \in \mathcal A}  \Pr(S_{t-1} = s_{t-1} | \pi = \pi_\theta) \pi_\theta(a_{t-1} | s_{t-1}) P(s_t | s_{t-1}, a_{t-1}) \\
=& \sum_{s_{t-1} \in \mathcal S} \sum_{a_{t-1} \in \mathcal A}  \Big(  \pi_\theta(a_{t-1} | s_{t-1}) P(s_t | s_{t-1}, a_{t-1}) \frac{\partial}{\partial \theta_i} \Pr(S_{t-1} = s_{t-1} | \pi = \pi_\theta) \Big)\\
 &+ \sum_{s_{t-1} \in \mathcal S} \sum_{a_{t-1} \in \mathcal A} \Big(\Pr(S_{t-1} = s_{t-1} | \pi = \pi_\theta) P(s_t | s_{t-1}, a_{t-1}) \frac{\partial}{\partial \theta_i}  \pi_\theta(a_{t-1} | s_{t-1}) \Big) \\
 \leq & \sum_{s_{t-1} \in \mathcal S} \sum_{a_{t-1} \in \mathcal A}\frac{\partial}{\partial \theta_i} \Pr(S_{t-1} = s_{t-1} | \pi = \pi_\theta) + \sum_{s_{t-1} \in \mathcal S} \sum_{a_{t-1}} \frac{\partial}{\partial \theta_i}  \pi_\theta(a_{t-1} | s_{t-1}) \\
 \leq& \vert \mathcal S \vert \vert \mathcal A \vert L_{t-1} + \vert \mathcal S \vert \vert \mathcal A \vert L_\pi.
\end{align*}

If the partial derivative with respect to each parameter is bounded, then the partial derivative as a whole is bounded.
Notice that for all $s_0$: $$\left \Vert \frac{\partial}{\partial \theta} \Pr(S_{0} = s_0 | \pi = \pi_\theta) \right \Vert = 0.$$ Therefore, by induction $L_t$ exists and is finite for all $t \leq T$.
\end{proof}

The bound given above is crude, but nevertheless, the existence of a constant bound $L_t$ at each timestep $t$ is sufficient for our purposes.
We then have:

\begin{cor}
\label{cor:bounded}
There exists some finite Lipschitz constant $L_d$ for the undiscounted state distribution ``$d$'' such that for all $s \in \mathcal S$:
\begin{align}
\left \Vert \frac{\partial}{\partial \theta} \sum_{t=1}^{T-1} \Pr(S_t = s | \pi = \pi_\theta) \right \Vert &\leq L_d.
\end{align}
\end{cor}
\begin{proof}
For all $s$:
\begin{align*}
\left \Vert \frac{\partial}{\partial \theta} \sum_{t=1}^{T-1} \Pr(S_t = s | \pi = \pi_\theta) \right \Vert &\leq \sum_{t=1}^{T-1} \left \Vert \frac{\partial}{\partial \theta}  \Pr(S_t = s | \pi = \pi_\theta) \right \Vert \\
&\leq \sum_{t=1}^{T-1} L_t.
\end{align*}
\end{proof}

We now have everything we need to prove that an upper bound on the error vector that is proportional to $(1 - \gamma)$, which will allow us to derive our convergence results.

\begin{lemma}
\label{lem:error}
There exists some Lipschitz constant $L_e$ for the error vector ``$e$'' such that for all $\gamma$:
\begin{equation}
\left \Vert \sum_{s \in \mathcal S} \dvf(s) \frac{\partial}{\partial \theta}  \ddf(s) \right \Vert \leq (1 - \gamma) L_e.
\end{equation}
\end{lemma}
\begin{proof}
We have:
\begin{align*}
\left \Vert \sum_{s \in \mathcal S} V^\theta_{\gamma}(s) \frac{\partial}{\partial \theta}  d^\theta_{\gamma}(s) \right \Vert
& \leq \vert \mathcal S \vert \max_{s \in \mathcal S} \left\Vert V^\theta_{\gamma}(s)  \frac{\partial}{\partial \theta}  d^\theta_{\gamma}(s) \right\Vert \\
& \leq \vert \mathcal S  \vert V_\text{max} \max_{s \in \mathcal S} \left \Vert \frac{\partial}{\partial \theta} d^\theta_{\gamma}(s) \right \Vert \\
& = \vert \mathcal S  \vert V_\text{max} \max_{s \in \mathcal S} \left \Vert \frac{\partial}{\partial \theta} (1 - \gamma) \sum_{t = 1}^{T - 1} \Pr(S_t = s | \pi = \pi_\theta) \right \Vert \\
& = (1 - \gamma) \vert \mathcal S  \vert V_\text{max} \left \Vert \frac{\partial}{\partial \theta} \sum_{t = 1}^{T - 1} \Pr(S_t = s | \pi = \pi_\theta) \right \Vert   \\
& \leq (1 - \gamma) \vert \mathcal S  \vert V_\text{max} L_d, \\
\end{align*}
where $V_\text{max} = \sum_{t=0}^T R_\text{max}$ and $L_d$ is given by Corollary \ref{cor:bounded}. 
\end{proof}

\section{Convergence of the Discounted Approximation}

We now proceed to our main result, which shows that slowly increasing $\gamma$ over time while following the discounted approximation results in a locally optimal policy under certain conditions.

\begin{thm}
Let $\theta_i$ be the sequence generated by the method

\begin{equation}
    \theta_{i+1} = \theta_i + \alpha_i \hat \nabla(\theta_i, \gamma_i).
\end{equation}

Assume that for all $i$ the step size $\alpha_i > 0$, the discount factor $\gamma_i \in [0, 1]$, and the following conditions are satisfied:

\begin{align}
    &\sum_{i=0}^\infty \alpha_i = \infty,& &\sum_{i=0}^\infty \alpha_i^2 < \infty,&
    &\alpha_i \geq c (1 - \gamma_i),&
\end{align}

where $c$ is some constant.
Then $J(\theta_i)$ converges to a finite value and $\lim_{i \to \infty} \nabla J(\theta_i) = 0$.
Furthermore, every limit point of $\theta_i$ is a stationary point of $J$.
\end{thm}

\begin{proof}
We showed in Equation \ref{eq:bias} that the discounted approximation can be written in terms of an ascent direction with respect to $J$ and an error term.
Therefore, we can apply Theorem \ref{thm:convergence} if we can show that the relevant assumptions hold.
Assumption \ref{ass:lipschitz} is satisfied by the parameterization of $\pi_\theta$ and because the MDP is finite. 
Assumption \ref{ass:step-size} is trivially satisfied by the statement of the theorem.
Assumption \ref{ass:ascent} is satisfied as $\nabla J(\theta)$ is exactly the steepest ascent direction on $J$.
Therefore, we need only to prove that Assumption \ref{ass:error} is satisfied. We have:

From Lemma \ref{lem:error}, we have at iteration $i$:

\begin{align*}
\left \Vert \sum_{s \in \mathcal S} V^\theta_{\gamma_i}(s) \frac{\partial}{\partial \theta}  d^\theta_{\gamma_i}(s) \right \Vert &\leq (1 - \gamma_i) \vert \mathcal S  \vert V_\text{max} L_d \\
&\leq \alpha_i \vert \mathcal S  \vert V_\text{max} L_d,
\end{align*}

thus satisfying Assumption \ref{ass:error} with parameters $p =  \vert \mathcal S  \vert V_\text{max} L_d$ and $q = 0$.

\end{proof}

\section{Conclusions}

We proved in Lemma \ref{lem:error} that the error in the discounted approximation is upper bounded by some finite value proportional to $(1-\gamma)$, thus providing a reasonable justification for the use of the discounted approximation, especially with high values of $\gamma$.
In Theorem \ref{thm:convergence}, we proved that convergence to the optimal policy is guaranteed when we increase $\gamma$ at a certain minimum rate.
These results help clarify the role of discounting in policy gradient methods and provide a solid theoretical foundation for the use of the discounted approximation in conjunction with strategies for increasing $\gamma$ over time.
As the discounted approximation is widely used in practice \citep{nota2019policy}, these results have significant implications for improving the convergence of practical algorithms.

\bibliographystyle{unsrtnat}
\bibliography{main}
\end{document}

%% file: preamble.tex
\usepackage[parfill]{parskip} 
\usepackage{courier} 

\usepackage{graphicx}
\usepackage{subcaption}
\usepackage[space]{grffile} 

\usepackage[boxruled,linesnumbered,vlined,inoutnumbered]{algorithm2e}

\usepackage{lipsum}

\usepackage[round]{natbib}

\usepackage[breaklinks]{hyperref}

\usepackage{xcolor}
\definecolor{dark-red}{rgb}{0.4,0.15,0.15}
\definecolor{dark-blue}{rgb}{0,0,0.7}
\hypersetup{
    colorlinks, linkcolor={dark-blue},
    citecolor={dark-blue}, urlcolor={dark-blue}
}

\newcommand{\dvf}{V^{\pi_\theta}_\gamma}

\newcommand{\ddf}{d^{\pi_\theta}_\gamma}

%% file: math.tex
\usepackage{amsmath}
\usepackage{amsthm}
\usepackage{amssymb}
\usepackage{mathtools}
\usepackage{mathrsfs}

\usepackage{bm}

\newtheorem{thm}{Theorem}[]

\newtheorem{lemma}{Lemma}[]

\newtheorem{ass}{Assumption}[]
\newtheorem{cor}{Corollary}[]
\usepackage{theoremref} 

\usepackage{amsmath,amsfonts,bm}

\def\eqref#1{equation~\ref{#1}}

\def\1{\bm{1}}

\DeclareMathAlphabet{\mathsfit}{\encodingdefault}{\sfdefault}{m}{sl}
\SetMathAlphabet{\mathsfit}{bold}{\encodingdefault}{\sfdefault}{bx}{n}

\newcommand{\E}{\mathbb{E}}

